\pdfoutput=1

\documentclass[11pt]{article}

\usepackage{acl}

\usepackage{bbm}
\usepackage[disable]{todonotes}
\makeatletter
\newcommand*\iftodonotes{\if@todonotes@disabled\expandafter\@secondoftwo\else\expandafter\@firstoftwo\fi}  %
\makeatother

\definecolor{dandelion}{HTML}{FFD464}

\definecolor{bittersweet}{HTML}{C04F17}

\definecolor{mintgreen}{RGB}{152, 255, 152}

\usepackage{tipa}

\usepackage{times}
\usepackage{latexsym}
\usepackage{fontawesome}
\usepackage{siunitx}
\usepackage{xspace}
\usepackage[inline]{enumitem}
\usepackage[multiple]{footmisc}
\usepackage{xcolor}
\usepackage{mleftright}
\usepackage{emoji}

\usepackage[T1]{fontenc}

\usepackage[utf8]{inputenc}

\usepackage{microtype}

\usepackage{inconsolata}

\usepackage[noend]{algpseudocode}
\usepackage{algorithm}
\usepackage{amsfonts}
\usepackage{amsmath}
\usepackage{amssymb}
\usepackage{bm}
\usepackage{cleveref}
\usepackage{xspace}
\usepackage{amsthm}
\usepackage{mathtools}
\usepackage{relsize}
\usepackage{caption,subcaption,booktabs}
\usepackage{import}
\usepackage{booktabs}
\usepackage{graphicx}
\usepackage{tikz}
\usepackage{tabularx}
\usetikzlibrary{positioning}
\usetikzlibrary{backgrounds}
\usetikzlibrary{calc}

\usepackage{xurl}

\newcommand{\modelname}[1]{{\tt{#1}}}

\crefname{section}{\S}{\S\S}
\Crefname{section}{\S}{\S\S}
\crefname{table}{Tab.}{}
\crefname{figure}{Fig.}{}
\crefname{algorithm}{Algorithm}{}
\crefname{equation}{Eq.}{Eqs.}  %
\crefname{appendix}{App.}{}
\crefname{thm}{Theorem}{Theorems}
\crefname{prop}{Proposition}{Propositions}
\crefname{cor}{Corollary}{Corollaries}
\crefname{observation}{Observation}{Observations}
\crefname{assumption}{Assumption}{Assumptions}
\crefformat{section}{\S#2#1#3}

\theoremstyle{definition}

\newtheorem{proposition}{Proposition}[section]

\newcommand{\defn}[1]{{\textbf{#1}}}

\newcommand{\defequals}{\triangleq}
\newcommand{\mtrian}{\mathrel{\raisebox{-0.1ex}{%
\scalebox{0.8}[0.6]{$\vartriangle$}}}}
\newcommand{\defpropto}{\overset{\mtrian}{\propto}}

\usetikzlibrary{shapes.misc}

\newcommand{\R}{\mathbb{R}}

\newcommand{\RD}{\mathbb{R}^D}

\newcommand{\kleene}[1]{{#1^*}}
\newcommand{\STR}{{\kleene{\Sigma}}}

\newcommand{\spearman}{{\rho}}
\newcommand{\pearson}{{r}}

\definecolor{MyTawny}{HTML}{d55e00} %
\definecolor{MyGreen}{HTML}{029e73}
\definecolor{MyBlue}{HTML}{0173b2}
\definecolor{MyOrange}{HTML}{de8f05}
\definecolor{MyBronze}{HTML}{ca9161}
\definecolor{MySilver}{HTML}{949494}
\definecolor{MyKerria}{HTML}{F8B500}

\newcommand{\querytext}[1]{{\color{MyGreen} \textit{#1}}}
\newcommand{\contexttext}[1]{{\color{MyOrange} \textit{#1}}}
\newcommand{\answertext}[1]{{\color{MyTawny} \textit{#1}}}

\newcommand{\realtext}[1]{{\color{MyBronze} \textit{#1}}}
\newcommand{\faketext}[1]{{\color{MySilver} \textit{#1}}}

\newcommand{\opentext}[1]{{\color{MyGreen} \textit{#1}}}
\newcommand{\closedtext}[1]{{\color{MyTawny} \textit{#1}}}

\newcommand{\context}{{\color{MyOrange}{c}}}

\newcommand{\rvContext}{{\color{MyOrange}{C}}}
\newcommand{\rvAnswer}{{\color{MyTawny}{A}}}

\newcommand{\rvQuery}{{\color{MyGreen}{Q}}}

\newcommand{\numcontextsample}{{\color{MyOrange}{n_c}}}

\newcommand{\order}[1]{\ensuremath{\mathcal{O}\left(#1\right)}}
\newcommand{\bigo}[1]{{\order}}

\newcommand{\alphabet}{\Sigma}

\newcommand{\lm}{p_{\textsc{m}}}

\newcommand{\answer}{{\color{MyTawny}a}}
\newcommand{\query}{{\color{MyGreen}{q}}}

\newcommand{\stickyscoresymb}{\chi}
\newcommand{\intrinsicsussymb}{\chi^*}

\newcommand{\MI}{\mathrm{I}}
\newcommand{\fimsymb}{\mathcal{J}}
\newcommand{\lmembed}{\boldsymbol{\theta}}
\newcommand{\deltaembedding}{\boldsymbol{\delta}}
\newcommand{\funcp}{f_{\query}}
\newcommand{\eye}{\mathbf{I}}
\DeclareMathOperator{\mTr}{Tr}
\newcommand{\Tr}[1]{\mTr\mleft(#1\mright)}

\newcommand{\expectation}[2]{\mathop{{}\mathbb{E}}_{#1}\mleft[ #2 \mright]}
\newcommand{\variance}[2]{\mathop{{}\mathbb{V}}_{#1}\mleft[ #2 \mright]}

\newcommand{\variancedelta}{\mathbf{S}}

\DeclarePairedDelimiterX{\infdivx}[2]{(}{)}{#1\;\delimsize\Vert\;#2}
\newcommand{\kldiv}{D_{\mathrm{KL}}\infdivx}

\definecolor{red_fig}{HTML}{D95847}
\definecolor{blue_fig}{HTML}{5D7CE6}

\definecolor{skipcolor}{HTML}{CC78BC}
\definecolor{valuecolor}{HTML}{029E73}
\definecolor{querycolor}{HTML}{DE8F05}
\definecolor{keycolor}{HTML}{0173B2}

\newcommand{\entityexample}[1]{{\color{MyBlue}{\textit{\textbf{#1}}}}}

\title{Efficiently Computing Susceptibility to Context in Language Models}
\author{
Tianyu Liu~\;~Kevin Du~\;~Mrinmaya Sachan~\;~Ryan Cotterell \\
\normalsize 
\{\href{mailto:tianyu.liu@inf.ethz.ch}{\texttt{tianyu.liu}}, \href{mailto:kevidu@inf.ethz.ch}{\texttt{kevin.du}}, \href{mailto:mrinmaya.sachan@inf.ethz.ch}{\texttt{mrinmaya.sachan}}, \href{mailto:ryan.cotterell@inf.ethz.ch}{\texttt{ryan.cotterell}}\}\texttt{@inf.ethz.ch} \\
\setlength{\fboxsep}{2.5pt}%
\setlength{\fboxrule}{2.5pt}%
\fcolorbox{white}{white}{
    \includegraphics[width=.15\linewidth]{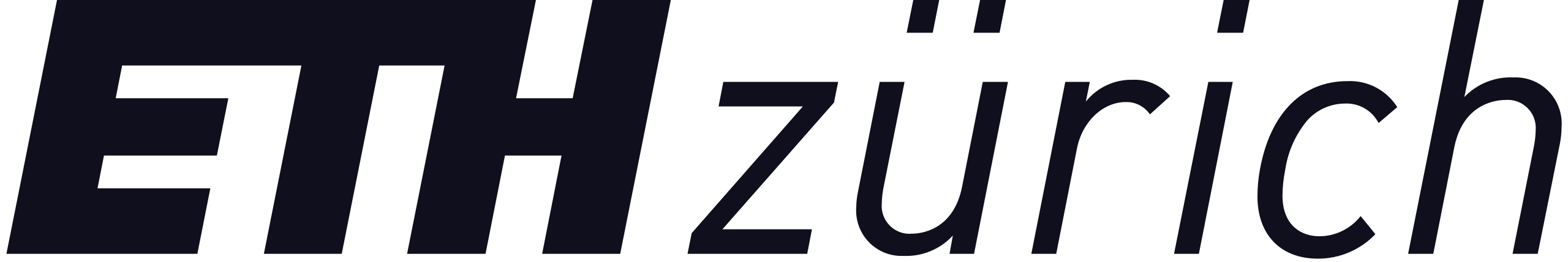}
}
}

\begin{document}
\maketitle
\begin{abstract}
One strength of modern language models is their ability to incorporate information from a user-input context when answering queries. 
However, they are not equally sensitive to the subtle changes to that context.
To quantify this, \citet{du2024context} gives an information-theoretic metric to measure such sensitivity. 
Their metric, susceptibility, is defined as the degree to which contexts can influence a model's response to a query at a distributional level.
However, exactly computing susceptibility is difficult and, thus, \citet{du2024context} falls back on a Monte Carlo approximation.
Due to the large number of samples required, the Monte Carlo approximation is inefficient in practice. 
As a faster alternative, we propose Fisher susceptibility, an efficient method to estimate the susceptibility based on Fisher information.
Empirically, we validate that Fisher susceptibility is comparable to Monte Carlo estimated susceptibility across a diverse set of query domains despite its being $70\times$ faster.
Exploiting the improved efficiency, we apply Fisher susceptibility to analyze factors affecting the susceptibility of language models.
We observe that larger models are as susceptible as smaller ones.\footnote{Our code is available at \url{https://github.com/lyutyuh/susceptibility}.}
\end{abstract}

\section{Introduction}
Much of current language models' (LM) capabilities and success are due to their responsiveness to different user-input contexts, e.g., prompts, and ability to integrate those contexts with prior knowledge \cite[][\emph{inter alia}]{NEURIPS2020_1457c0d6,bubeck2023sparks}.
Given this ability, it is natural for us to wonder how easily a model's prior belief can be changed by an input context.
For example, a language model might correctly complete the prompt \querytext{Here Comes the Sun is performed by} with \answertext{The Beatles}, but wrongly when this prompt is prepended with a random context such as \contexttext{Falafel wraps make ostriches burp.}; such a context might distract the model away from its original behavior.

Recently, \citet{du2024context} studied how easily input contexts can skew language models to give answers to queries that are different than the answers stored in their prior knowledge from an information-theoretic point of view.
They propose a metric, termed \textit{susceptibility}, to quantify the discrepancy between a model's prior belief and updated belief after seeing the input contexts as the mutual information between input contexts and answers from a specific language model. 
This metric makes use of the language model's \emph{distribution} over answers rather than relying on an approximate argmax, as is common in previously proposed metrics \citep{shi-etal-2023-irrelevant,wang2024resilience}.
Susceptibility, as an information-theoretic metric, captures changes in model behavior that might not necessarily surface while decoding an answer.

Computing susceptibility requires, in principle, summing over all possible answers, a countably infinite set.
In the general case, the authors know of no efficient algorithm to perform such a summation; \citet{du2024context} propose a Monte Carlo approximation. 
Such a sampling approximation, however, is computationally expensive: to compute the susceptibility for one query, one needs to execute a forward pass of the neural language model for each context in the sample set, making it inefficient to be applied to large-scale datasets.
For example, \citet{du2024context} considers sampling 600 contexts for every individual query and thus requires 600 forward passes.
In light of the computation required for the Monte Carlo approximation to susceptibility, we propose a more efficient approximation based on Fisher information that does not require sampling to estimate the susceptibility; we term this approximation \textit{Fisher susceptibility}.\looseness=-1

We conduct experiments across queries from 122 relation domains (e.g., \querytext{alumniOf} or \querytext{capitalOf}) in the YAGO knowledge graph \cite{yago2007} with a variety of model sizes (e.g., from 70 million to 8 billion parameters), language model families (e.g., \modelname{Pythia} \citep{biderman2023pythia}, \modelname{LLaMA} \citep{touvron2023llama}, and \modelname{GPT-2} \citep{NEURIPS2020_1457c0d6}), and fine-tuning schemes (e.g., instruction-tuning).
First, to empirically validate Fisher susceptibility, we show that it is tightly correlated with \citeposs{du2024context} Monte Carlo approximation to susceptibility across domains while also benchmarking its speed.
In these experiments, we find that compared to the Monte Carlo estimation with a sample size of 256, Fisher susceptibility exhibits a $70\times$ improvement in runtime.
Then, we use the increased efficiency of Fisher susceptibility to investigate further initial questions like the one posed in the first paragraph: how susceptible are language models across different sizes and model families, and what factors influence a model's Fisher susceptibility for a query?
We find that larger models are not less susceptible than smaller ones and that instruction-tuning does not help reduce susceptibility.
We further find that queries about well-known entities are equally susceptible as less frequent ones under Fisher susceptibility, which contrasts with the finding from \citet{du2024context} that susceptibility is negatively related with expected familiarity when contexts might be relevant. 
To the extent that this is not an approximation error, this finding suggests that, regardless of how much prior knowledge it has about a query, a language model is still susceptible to contexts and can integrate new information from them.\looseness=-1

\section{Susceptibility to Context in LMs}
Language models are capable of answering a wide range of queries formulated in natural language, including code auto-completion, text generation, and factual question answering \cite[][\textit{inter alia}]{kwiatkowski_natural_2019,NEURIPS2020_1457c0d6,kasai_realtime_2022}.
When responding to a query, language models need to synthesize the prior knowledge they learned during pretraining with the new information provided in the input context \citep{kwiatkowski_natural_2019, joshi_triviaqa_2017, berant_semantic_2013, kasai_realtime_2022}.
For example, in the knowledge conflict setting proposed by \citet{longpre-etal-2021-entity} given the query \querytext{What's the capital of Ireland?} and context \contexttext{The capital of Ireland is Rome.}, the model must decide between whether to agree with its prior knowledge (\answertext{Dublin}) or the context (\answertext{Rome}).

How easily large language models are affected by contexts is a well-studied problem \cite[][\emph{inter alia}]{liang2023holistic,yoran2024making,wang2024resilience,wu2024easily,wu2024instructing}. 
Many studies, including \citet{longpre-etal-2021-entity, chen-etal-2022-rich, xie_adaptive_2023}, measure how easily context influences a model by computing the \defn{memorization ratio}: the proportion of examples for which the model maintains from before the context was introduced.
However, memorization rate may not fully capture the strength of this influence.
For example, adding a context could take a model's probability of answering a token from 95\% to 51\%, but the 1-best answer would remain the same, and the context's influence would not be detectable by memorization rate.
To solve this issue, \citet{du2024context} provides a metric that measures the influence of context on a model's answer to a query using a more fine-grained metric based on the model's full answer distribution, the \defn{susceptibility}, which we describe in detail in \Cref{sec:measure_sus}.
To compute susceptibility, \citet{du2024context} proposes a Monte Carlo approximation.
However, it requires running one forward pass, which is computationally expensive, per sampled answer; this limits the scale of analysis.
Thus, we aim to find a method that is more efficient, general, and interpretable.\looseness=-1

\section{Measuring Susceptibility}
\label{sec:measure_sus}
Let $\alphabet$ be an \defn{alphabet}, i.e., a finite, non-empty set.
A \defn{language model} $\lm$ over an alphabet $\alphabet$ is a distribution over $\STR$, the set of all strings with \defn{tokens} drawn from $\alphabet$.
We denote a query by $\query \in \STR$, and the answer to the query by $\answer \in \STR$.
When querying a language model, a context $\context \in \STR$ is provided together with the query $\query$ in one input sequence $\context \oplus \query$ where $\oplus$ denotes string concatenation. 
In practice, querying a language model is the process of generating an answer $\answer$ from the distribution $\lm(\cdot \mid \context \oplus \query)$.\looseness=-1

\subsection{Susceptibility as Mutual Information}
For a pretrained language model, \citet{du2024context} gives a metric that quantifies how easily a language model's distribution given a query is altered by a context.
To investigate how the answer distributions of a language model respond to a class of contexts, they consider three $\STR$-valued random variables, $\rvContext$, $\rvQuery$ and $\rvAnswer$, standing for context, query, and answer.
For a specific query $\query$, $\rvContext$ and $\rvAnswer$ are jointly distributed according to the distribution
\begin{align}\label{eq:joint}
   p(\rvContext = \context, \rvAnswer = \answer \mid \rvQuery = \query) \defpropto \lm( \answer \mid \context \oplus \query).
\end{align} 
Based on the joint distribution in \Cref{eq:joint}, \citet{du2024context} defines the
\defn{susceptibility} of a query $\query$ as the conditional mutual information between $\rvContext$ and $\rvAnswer$:\looseness=-1
\begin{align}
& \stickyscoresymb(\query) 
    \defequals \MI(\rvContext; \rvAnswer \mid \rvQuery = \query) \label{eq:susceptibility} \\
    &=\expectation{\rvContext}{\kldiv{p(\rvAnswer \mid \rvContext = \context, \rvQuery = \query)}{p(\rvAnswer \mid \rvQuery = \query)} }  \nonumber
\end{align}
See \citet[\S 3.2,][]{du2024context} for additional details. 

Intuitively, the susceptibility measures how much the answer distribution differs before and after being prompted with additional contexts \emph{on average}.
By taking the expectation over all possible contexts, we arrive at the susceptibility. 
If the model is not susceptible, the Kullback--Leibler divergence $\kldiv{p(\rvAnswer \mid \rvContext, \rvQuery = \query)}{p(\rvAnswer \mid \rvQuery = \query)}$ should be $0$ in expectation, meaning that $p(\rvAnswer \mid \rvContext, \query)$ and $p(\rvAnswer \mid \rvQuery = \query)$ are the same, regardless of the value of $\rvContext$.\looseness=-1

\subsection{Computational Cost} \label{sec:mi-problem}
\citet{du2024context} gives a practical algorithm to estimate susceptibility based on Monte Carlo estimation.
We call susceptibility estimated in this manner \defn{Monte Carlo susceptibility}.
However, computing Monte Carlo susceptibility requires sampling a large number of input--output pairs from a language model.
Indeed, we first sample contexts from a distribution $p(\rvContext)$, and then evaluate the conditional distribution over answer $p(\rvAnswer \mid \rvContext = \context, \rvQuery = \query)$ for all answers $\answer$ in the answer space, which requires one forward pass per sample.
Because the answer space is typically all of $\STR$, $p(\rvAnswer \mid \rvContext = \context, \rvQuery = \query)$ is additionally approximated by considering the next-token distribution \cite{du2024context}. 
Thus, the runtime of such a scheme, with the next-token approximation, is still $\mathcal{O}\left(\numcontextsample \times |\alphabet|\right)$ where $\numcontextsample$ is the number of context samples; for reference, \citet{du2024context} takes $600$ Monte Carlo samples.\looseness=-1

\section{Fisher Susceptibility}
\label{sec:formalization}
To alleviate the cost of the Monte Carlo approximation discussed in \Cref{sec:mi-problem}, we derive an efficient approximation based on Fisher information.\looseness=-1

\begin{figure*}
    \centering
    \includegraphics[width=0.98\textwidth]{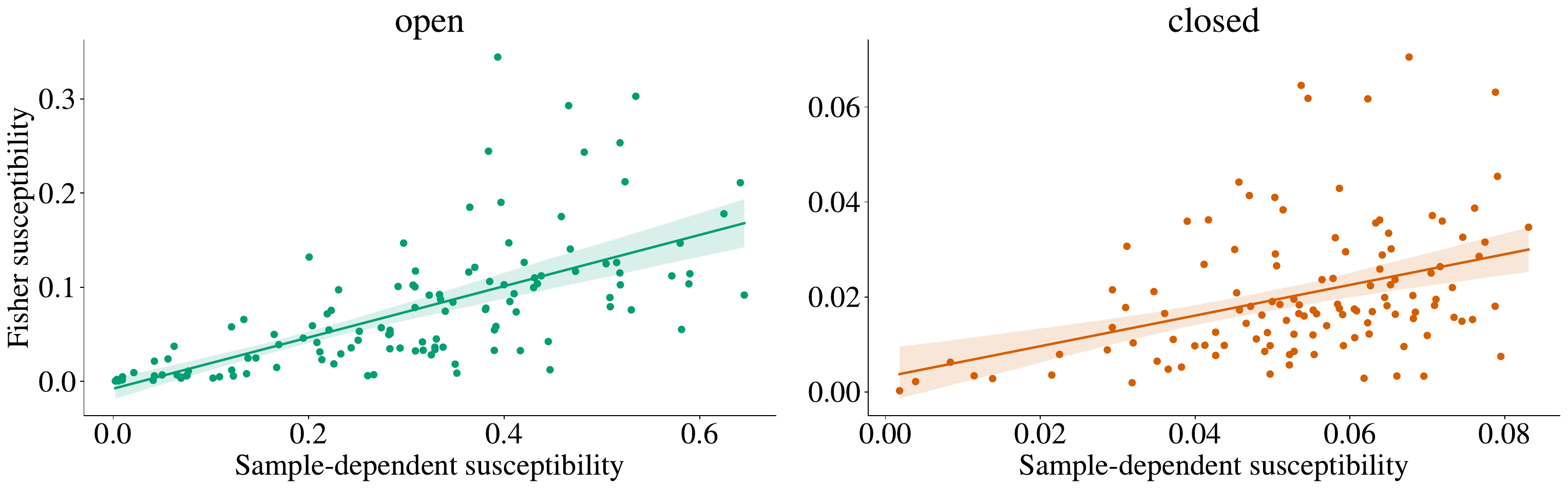}
    \caption{This plot shows the average susceptibility on 122 relation domains (e.g., \querytext{alumniOf}) in the YAGO dataset. Each point represents the average score of queries on a relation domain. The $x$-coordinate represents Monte Carlo susceptibility and $y$-coordinate represents Fisher susceptibility. The scores are computed with \modelname{LLaMA-3-8B-Instruct}. For both query types, the two metrics are strongly correlated.}
    \label{fig:sus-vs-fisher-corpus}
\end{figure*}

\subsection{A Simple Reparameterization} \label{sec:reparameterization}

First, we give a simple reparameterization that applies to any neural $\lm$.
Given a query $\query$, each context $\context \in \STR$ defines a probability distribution over $\STR$, viewed as answers to the query $\query$.
Now, let $\lmembed \colon \STR \to \RD$ be an injective embedding function that maps strings to unique real vectors. 
We can view $\lmembed$ as an \emph{index} and, thus, consider the language model as a parameterized family of distributions with a real-valued index:
\begin{equation}
\!\!\lm(\rvAnswer = \answer \mid \rvContext = \context, \rvQuery=  \query) \defequals \funcp(\answer; \lmembed(\context \oplus \query)),
\end{equation}
In this view, the conditional language model is parameterized by a real-vector rather than by a string that encodes the context.
In our experiments, we define the embedding function to map its argument $\context \oplus \query$ to be the concatenation of the real-valued embedding vectors of $\query$ and the $\context$ given by the pretrained language model. %

\subsection{The Fisher Information Matrix}
The reparameterization given in \Cref{sec:reparameterization} opens up a new type of approximation.
Specifically, we can now define the following \defn{Fisher information}
\begin{align} \label{eq:fisher-information}
    \fimsymb & \mleft( \lmembed(\query) \mright) \defequals \\ 
    & \expectation{\answer \sim \funcp(\cdot; \lmembed(\query)) }{ \frac{\partial\log \funcp(\answer; \lmembed(\query)) }{ \partial \lmembed } \frac{\partial\log \funcp(\answer; \lmembed(\query)) }{ \partial \lmembed^\top} }\nonumber
\end{align}
as a measure of how much influence a context has on the distribution over answers. 
The Fisher information matrix can be interpreted as a quantification of the amount of information that an observable random variable carries about an unknown parameter \cite[][]{LehmCase98}: 
The higher the Fisher information is, the easier we can estimate the unknown parameters of the distribution from samples.
In our context, we are interested in how language models react to the change in the \emph{context} rather than the pretrained parameters of the model.
Thus, we treat the parameters of the language model itself as constants and investigate the Fisher information of $\funcp$ with respect to the embedding vector $\lmembed(\query)$ of query $\query$.

More relevant to susceptibility, there is a formal relationship between the Fisher information matrix and the KL divergence.
Performing a second-order Taylor expansion of the KL divergence, we arrive at\looseness=-1
\begin{equation}
\begin{aligned}
     & \kldiv{\funcp(\cdot \mid \lmembed(\query) + \deltaembedding)}{\funcp(\cdot \mid \lmembed(\query))}  \\
    & \qquad \qquad = \,\, \frac{1}{2} \deltaembedding^\top \fimsymb \mleft( \lmembed(\query)  \mright) \deltaembedding + \order{||\deltaembedding||^3}, \label{eq:fisher-quadratic-form}
\end{aligned}
\end{equation}
where $\deltaembedding \in \mathbb{R}^D$ is a perturbation of the distribution parameter. 
When the perturbation is small, i.e., when $|| \deltaembedding ||^3$ is small, we expect \Cref{eq:fisher-quadratic-form} to be dominated by its first term.
This view invites a simple approximation of the KL divergence. \looseness=-1

\subsection{Fisher and Mutual Information} \label{sec:susceptibility_score}

In \cref{eq:susceptibility}, the susceptibility $\stickyscoresymb(\query)$ is defined as the mutual information between the context and answer random variables, conditioned on a fixed query about an entity.
Using the identity given in \cref{eq:fisher-quadratic-form}, the susceptibility can be rewritten as
\begin{subequations}\label{eq:sus-fi}   
\begin{align}
\stickyscoresymb(  \query)  & \defequals \MI(\rvContext; \rvAnswer \mid \rvQuery = \query)  \\ 
 & = \frac{1}{2} \expectation{\rvContext  }{\deltaembedding(\context, \query)^\top  \fimsymb( \lmembed(\query)) \deltaembedding(\context, \query)}\\
 & \qquad\qquad\qquad + \order{\expectation{\rvContext}{||\deltaembedding(\context, \query)||^3}} \nonumber \\
 &\approx \frac{1}{2} \expectation{\rvContext  }{\deltaembedding(\context, \query)^\top  \fimsymb( \lmembed(\query)) \deltaembedding(\context, \query)}, \label{eq:final-approx}
\end{align}
\end{subequations}
where we define the context-specific perturbation
\begin{equation}
    \deltaembedding(\context, \query) \defequals \lmembed(\context \oplus \query) - \lmembed(\query).
\end{equation}
The full derivation is given in \cref{app:derivation-sus-fi}.

\begin{table*}
    \renewcommand{\arraystretch}{1}
    \centering
    \small
    \resizebox{0.9\textwidth}{!}{%
    \begin{tabular}{lcccccc}
        \toprule
         & \multicolumn{2}{c}{Open} & \multicolumn{2}{c}{Closed} & \multicolumn{2}{c}{Overall}  \\ \cmidrule(lr){2-3} \cmidrule(lr){4-5} \cmidrule(lr){6-7} 
        Model & Pearson & Spearman & Pearson & Spearman & Pearson & Spearman  \\ \midrule
        \modelname{gpt2}   & 0.55   & 0.58   & 0.21   & 0.18   & 0.63   & 0.56    \\
\modelname{gpt2-large}   & 0.27   & 0.37   & 0.24   & 0.32   & 0.43   & 0.64    \\
\modelname{gpt2-medium}   & 0.42   & 0.47   & 0.29   & 0.28   & 0.57   & 0.68    \\
\modelname{gpt2-xl}   & 0.38   & 0.50   & 0.31   & 0.36   & 0.54   & 0.71    \\ \midrule
\modelname{LLaMA-3-8B}   & 0.46   & 0.66   & 0.31   & 0.35   & 0.47   & 0.65    \\
\modelname{LLaMA-3-8B-Instruct}   & 0.51   & 0.76   & 0.38   & 0.47   & 0.53   & 0.68    \\
\modelname{LLaMA-2-7B}   & 0.46   & 0.71   & 0.25   & 0.27   & 0.52   & 0.69    \\
\modelname{LLaMA-2-7B-chat}   & 0.55   & 0.78   & 0.40   & 0.47   & 0.47   & 0.65    \\ \midrule
\modelname{Pythia-70m-deduped}   & 0.66   & 0.68   & 0.32   & 0.39   & 0.64   & 0.57    \\
\modelname{Pythia-160m-deduped}   & 0.44   & 0.53   & 0.26   & 0.32   & 0.51   & 0.63    \\
\modelname{Pythia-410m-deduped}   & 0.59   & 0.64   & 0.36   & 0.35   & 0.70   & 0.77    \\
\modelname{Pythia-1.4b-deduped}   & 0.45   & 0.55   & 0.20   & 0.19   & 0.52   & 0.69    \\
\modelname{Pythia-2.8b-deduped}   & 0.37   & 0.43   & 0.25   & 0.29   & 0.55   & 0.69    \\
\modelname{Pythia-6.9b-deduped}   & 0.34   & 0.42   & 0.26   & 0.28   & 0.56   & 0.72    \\ \bottomrule
    \end{tabular}
    }
    \caption{Correlations between Monte Carlo susceptibility and Fisher susceptibility across different models. For each model, we compute Monte Carlo susceptibility and Fisher susceptibility. Then, we evaluate the Pearson's and Spearman's correlation between them on open, closed, and all queries from YAGO.} 
    \label{tab:fisher-sus-comparison}
\end{table*}

To derive an approximation to susceptibility that does not require sampling, we need to remove the terms in \cref{eq:sus-fi} that depend on the random variable $\rvContext$.
If we assume that
\begin{subequations}
\begin{align} \label{eq:mean-and-variance-1}
    \expectation{\rvContext}{\deltaembedding(\context, \query)} &= \mathbf{m} \\
    \variance{\rvContext}{\deltaembedding(\context, \query)} &= \variancedelta,  \label{eq:mean-and-variance-2}
\end{align}
\end{subequations} 
we arrive at the following closed-form solution
\begin{equation} \label{eq:quadratic-expectation}
\begin{aligned}
 &\expectation{\rvContext}{\deltaembedding(\context, \query)^\top  \fimsymb( \lmembed(\query)) \deltaembedding(\context, \query)} \\
 &\quad = \Tr{\variancedelta\fimsymb( \lmembed(\query))} + \mathbf{m}^\top \fimsymb( \lmembed(\query)) \mathbf{m}
\end{aligned}
\end{equation}
by means of a well-known identity \citep{Petersen2008}.
If we take, for instance, $\mathbf{m}= \mathbf{0}$ and $\variancedelta = \eye$, \cref{eq:quadratic-expectation} simplifies to 
\begin{align} \label{eq:intrinsic-as-trace}
    \intrinsicsussymb(\query)
    \defequals  \Tr{\fimsymb(\lmembed(\query))},
\end{align}
which we term \defn{Fisher susceptibility}.

To the extent that \Cref{eq:intrinsic-as-trace} is a good approximation of \Cref{eq:sus-fi}, Fisher susceptibility $ \intrinsicsussymb(\query)$ should strongly correlate with \citeposs{du2024context} Monte Carlo approximation to $\stickyscoresymb(\query)$.
We remark again that Fisher susceptibility $\intrinsicsussymb(\query)$ does \emph{not} require computation of the distribution over contexts because, in \cref{eq:mean-and-variance-1,eq:mean-and-variance-2}, we made an assumption about the mean and variance of the perturbation vector $\deltaembedding$.
However, this additional assumption is not theoretically motivated, and, thus, we appeal to experimentation to vet the approximation.\looseness=-1

\subsection{Efficient Computation}
Recall that the Fisher information $\fimsymb(\lmembed(\query))$ is a matrix of shape $\R^{D \times D}$, where $D$ is the dimension of input embedding $\lmembed(\query)$.
Directly computing Fisher information matrix  $\fimsymb(\lmembed(\query))$ using \cref{eq:fisher-information} takes $\mathcal{O}(D^2)$; it has $D^2 \approx 10^{12}$ entries that need to be computed. 
Thus, we apply the following approximation \cite{du2024context} by truncating the distribution over answers: \looseness=-1
\begin{align}
     \fimsymb(\lmembed(\query)) &=\sum_{k=1}^{K} \funcp(\answer_k; \lmembed(\query))\label{eq:fisher-approximation}  \\ 
     & \,\,\,\quad \frac{\partial\log \funcp(\answer_k; \lmembed(\query)) }{ \partial \lmembed} \frac{\partial\log \funcp(\answer_k; \lmembed(\query)) }{  \partial \lmembed^\top}, \nonumber 
\end{align}
where $\{\answer_k\}_{k=1}^K$ are the top-$K$ highest-probability tokens.
This approximation reduces the number of entries to be computed from $\mathcal{O}(D^2)$ to $\mathcal{O}(KD$).
The individual gradients $\frac{\partial\log \funcp(\answer_k; \lmembed) }{ \partial \lmembed }$ can be efficiently computed with auto-differentiation in the same time complexity that it takes to compute a \emph{single} forward pass of the language model.

\section{Experiments}
Experimentally, we first aim to show that Fisher susceptibility correlates well with Monte Carlo susceptibility.
Then, we investigate what factors influence higher susceptibility, across models, queries, and entities, accelerated by the use of Fisher susceptibility.
Finally, we apply Fisher susceptibility to evaluate language models' susceptibility to contexts.\looseness=-1

\subsection{Experiment Setup}
\label{sec:exp_setup}

For all of our experiments, we use the same framework and dataset provided by \citet{du2024context}.\footnote{\url{https://github.com/kdu4108/measureLM}}
For each of the relations from the YAGO knowledge graph \citep{yago2007}, we have two closed query forms (yes--no questions) and two open query forms.
For each relation, we subsample 100 entities, half of which are real entities extracted from YAGO, and half of which are fake entity names generated by GPT-4, from their dataset of 1000 entities.
In total, we construct a collection of 48,800 queries.
For each of these queries, we compute Monte Carlo susceptibility as a point of comparison.
We take a sample of 256 contexts, 8 of which directly mention the queried entity, per query for \modelname{Pythia} and \modelname{GPT-2} models, and 128 contexts, 4 of which directly mention the queried entity, per query for \modelname{LLaMA} models.

We also compute Fisher susceptibility for each of these entities according to \Cref{eq:intrinsic-as-trace}.
We repeat these for models of different families (i.e., \modelname{Pythia} \cite{biderman2023pythia}, GPT-2 \cite{NEURIPS2020_1457c0d6}, and \modelname{LLaMA} \cite{touvron2023llama}), model sizes,\footnote{We choose 70m, 410m, 1.4b, 2.8b, 6.9b for \modelname{Pythia-deduped}, and small, medium, large, xl for \modelname{GPT-2}.} and training types (e.g., pretrained vs instruction-tuned) when applicable.
A full list of models can be found in \Cref{tab:fisher-sus-comparison}

\paragraph{Dataset.}

Following \citet{du2024context}, we use 122 relations from the YAGO knowledge graph \cite{yago2007}, such as \querytext{birthPlace}, \querytext{leader}, and \querytext{homeLocation}. 
For each relation, we sample $50$ real entities from YAGO and $50$ fake entities generated by GPT-4 \citep{openai_gpt4_2023}.\footnote{\texttt{gpt-4-1106-preview}, January 2024.}
Our open and closed queries are constructed from templates of both question-answering and sentence-completion forms, e.g., (closed, question-answering) \querytext{Q: Is \entityexample{\{answer\}} the capital of \entityexample{\{entity\}}? A:}, (open, question-answering) \querytext{Q: What is the capital of \entityexample{\{entity\}}? A:}, and (open, sentence-completion) \querytext{The capital of \entityexample{\{entity\}} is}.
We parameterize the templates with entities (both real and fake entities) and answers in their respective slots in the templates.
For instantiating contexts, we use the base template from \citet{du2024context}, e.g., \contexttext{The capital of \{entity\} is \{answer\}.}.
We parameterize these context templates with both real and fake entities (and answers to the queries, when applicable).
For each relation domain, we randomly sample $2000$ contexts.
For the Monte Carlo estimate of susceptibility, we further subsample 256 contexts for \modelname{Pythia} and \modelname{GPT-2} and 128 contexts for \modelname{LLaMA}.\looseness=-1

\paragraph{Models.}
In our experiments, we use the pretrained language models (\modelname{Pythia}, \modelname{GPT-2}, \modelname{LLaMA}) from Huggingface.\footnote{We use \modelname{Pythia} (\url{https://huggingface.co/collections/EleutherAI/pythia-scaling-suite-64fb5dfa8c21ebb3db7ad2e1}), \modelname{GPT-2} (\url{https://huggingface.co/openai-community/gpt2}), \modelname{LLaMA-2} (\url{https://huggingface.co/collections/meta-llama/llama-2-family-661da1f90a9d678b6f55773b}), and \modelname{LLaMA-3} (\url{https://huggingface.co/collections/meta-llama/meta-llama-3-66214712577ca38149ebb2b6}).}

\begin{figure}[!t]
    \centering
    \includegraphics[width=\columnwidth]{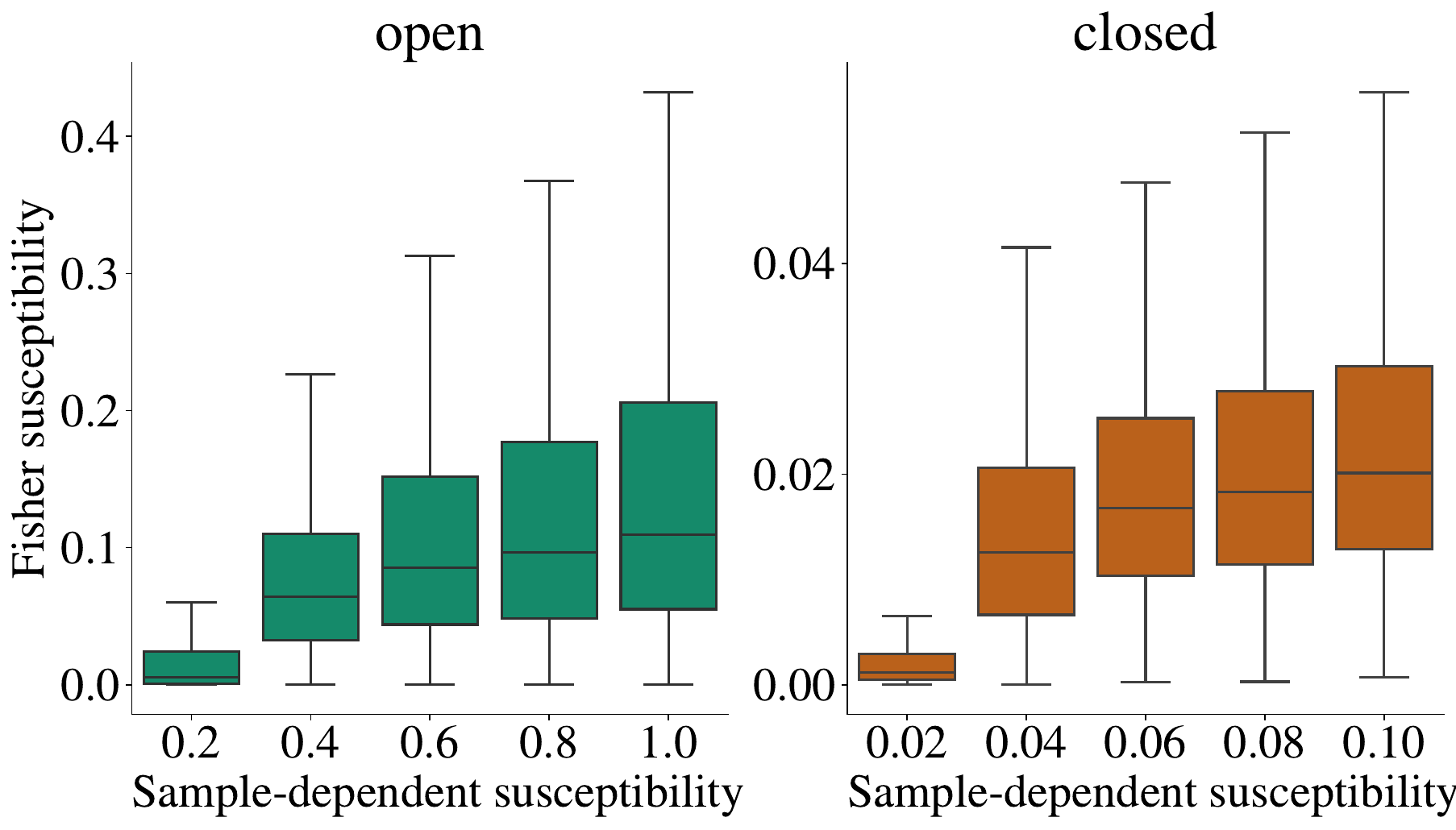}
    \caption{For both \opentext{open} and \closedtext{closed} queries, Fisher susceptibility $y$ strongly correlates with Monte Carlo susceptibility ($x$-axis). 
    Monte Carlo susceptibility is divided into 5 bins from $0$ to $1$ on \modelname{LLaMA-3-8B-Instruct}.}
    \label{fig:open-vs-closed-llama-3-instruct}
\end{figure}

\paragraph{Computational Resources.}
All experiments presented were conducted on two NVIDIA GeForce RTX 4090 GPUs with 24GB memory.
\modelname{LLaMA} models are stored and run in bfloat16 precision.
The other models are stored and run in float32 precision.\looseness=-1

\subsection{Comparing Susceptibility} \label{sec:corr_intrinsic}
We now demonstrate that Fisher susceptibility is a good approximation to susceptibility.\looseness=-1

\begin{table}
    \centering
    \resizebox{0.95\columnwidth}{!}{%
    \begin{tabular}{lccc}
        \toprule
         & \multicolumn{3}{c}{Query Type}  \\ \cmidrule(lr){2-4}
        Model & Overall & Closed & Open \\ \midrule
        \modelname{Pythia-70m-deduped}  & 0.30 & 0.03 & 0.54 \\
        \modelname{Pythia-160m-deduped} & 0.31 & 0.03 & 0.57 \\
        \modelname{Pythia-410m-deduped} & 0.31 & 0.04 & 0.56 \\
        \modelname{Pythia-1.4b-deduped} & 0.35 & 0.06 & 0.62 \\
        \modelname{Pythia-2.8b-deduped} & 0.33 & 0.03 & 0.58 \\
        \modelname{Pythia-6.9b-deduped} & 0.34 & 0.04 & 0.61 \\ \midrule
        \modelname{gpt2-small} & 0.34  & 0.08  & 0.57  \\
        \modelname{gpt2-medium} & 0.31  & 0.05  & 0.55  \\
        \modelname{gpt2-large} & 0.32  & 0.04  & 0.57  \\
        \modelname{gpt2-xl} & 0.29  & 0.04 & 0.51  \\ \midrule
        \modelname{LLaMA-2-7B}          & 0.34 & 0.07 & 0.59 \\
        \modelname{LLaMA-2-7B-chat}     & 0.57 & 0.23 & 0.61 \\
        \modelname{LLaMA-3-8B}          & 0.37 & 0.13 & 0.58 \\
        \modelname{LLaMA-3-8B-instruct} & 0.39 & 0.10 & 0.66 \\ \bottomrule
    \end{tabular}
    }
    \caption{Mean Monte Carlo susceptibility of different models on YAGO on different query types (open, closed, and overall).} \label{tab:model-sus-comparison}
\end{table}

\paragraph{Experiment Setup.}
We use the setup from \Cref{sec:exp_setup} to compute Fisher susceptibility.
For each query, we compute the Pearson's correlation and Spearman's correlation between Monte Carlo susceptibility and Fisher susceptibility.\looseness=-1

\begin{figure*}[!ht]
    \centering
    \includegraphics[width=\textwidth]{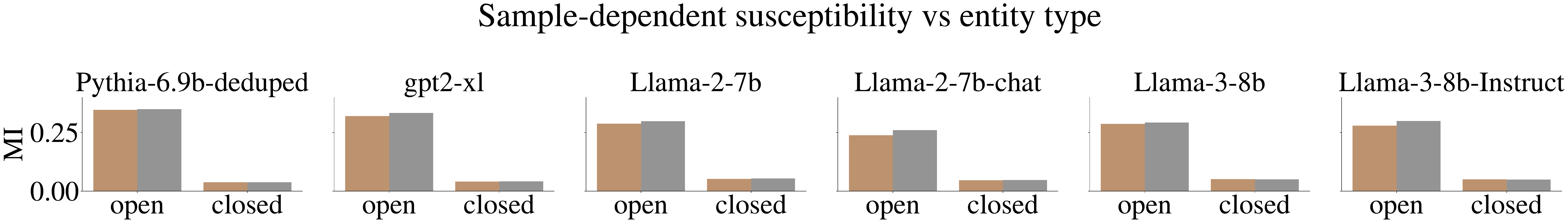}
    \includegraphics[width=\textwidth]{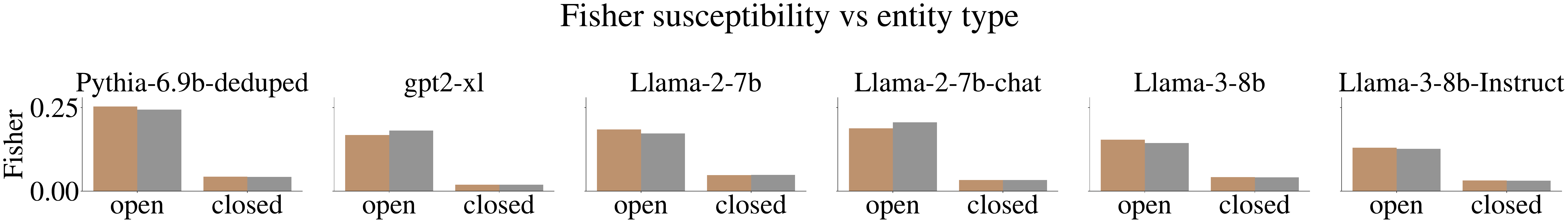}
    \caption{Susceptibility comparisons between open and closed queries for \realtext{real entities} (bars on the left) and \faketext{fake entities} (bars on the right). In each subplot, the two bars on the left represent open queries, and the two on the right represent closed queries. From this, we can see the susceptibility generally does not appear to differ much between real and fake entities. (Top) Monte Carlo susceptibility. (Bottom) Fisher susceptibility.}
    \label{fig:real-vs-fake}
\end{figure*}

\paragraph{Results.}
We compare Monte Carlo susceptibility and Fisher susceptibility. 
In \cref{fig:open-vs-closed-llama-3-instruct,tab:fisher-sus-comparison}, we observe a strong correlation between Fisher susceptibility and Monte Carlo susceptibility across all models.
Across queries, we find a Pearson's correlation of $\pearson = 0.51$ and a Spearman's correlation of $\spearman = 0.76$ on open queries, $\pearson = 0.38$ and $\spearman = 0.47$ on closed queries using \modelname{LLaMA-3-8B-instruct}. %
We also evaluate average susceptibility on the corpus level by averaging the susceptibility of queries in each domain in \cref{fig:sus-vs-fisher-corpus}. 
We measure a Pearson's correlation of $\pearson = 0.65$ and a Spearman's correlation of $\spearman = 0.76$ on open queries and $\pearson = 0.51, \spearman = 0.60$ on closed queries using \modelname{LLaMA-3-8B-instruct}.
We take these results as validation that Fisher susceptibility correlates with Monte Carlo susceptibility and is, thus, a good approximation. 
Moreover, the large Pearson's correlation coefficient indicates the relationship is linear.
We refer the readers to \Cref{app:additional-results} for full evaluation results on all sizes of models.\looseness=-1

\begin{figure*}
    \centering
    \includegraphics[width=\textwidth]{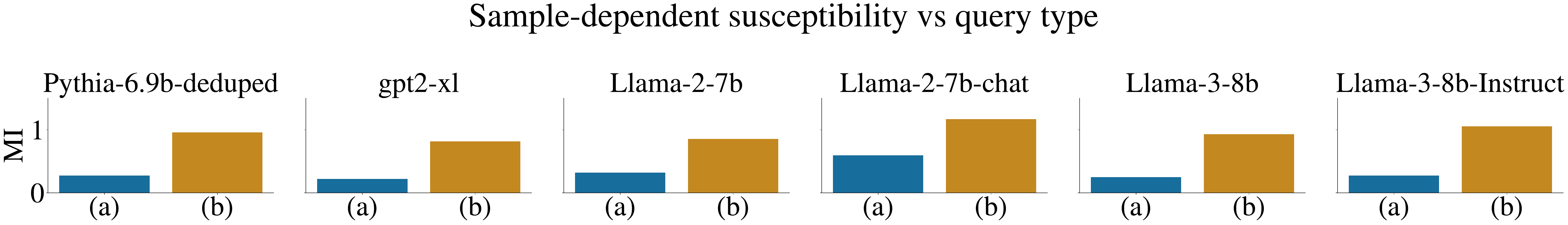}
    \includegraphics[width=\textwidth]{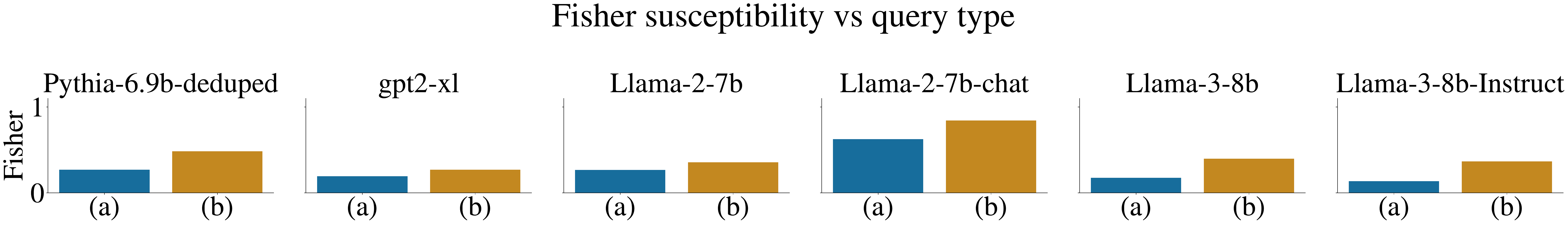}
    \caption{Susceptibility comparisons between (a) question-answering and (b) sentence-completion queries. 
    Consistently, question-answering formats are less susceptible for both susceptibility metrics.
    (Top) Monte Carlo susceptibility. (Bottom) Fisher susceptibility.}
    \label{fig:qa-vs-completion}
    \vspace{-10pt}
\end{figure*}

\paragraph{Runtime Comparison.}
We conduct an empirical analysis on the runtime of Fisher susceptibility and Monte Carlo susceptibility.
We find computing Fisher susceptibility is 70$\times$ faster when the number of samples for Monte Carlo susceptibility is chosen to be 256 and 30$\times$ faster when the number of samples is 128.
Specifically, for \modelname{LLaMA-3-8B} models, evaluating Monte Carlo susceptibility for all 48800 queries on YAGO costs 10 hours while computing Fisher susceptibility only costs 20 minutes. \looseness=-1

\subsection{Factors Affecting Fisher susceptibility}
We now aim to understand what factors could cause a language model to have higher susceptibility.
We investigate three aspects, namely the size and training method of language models, the format of the query, and the type of the entity.\looseness=-1

\paragraph{Models.} 
From \Cref{tab:model-sus-comparison}, we see that the susceptibility of the models of \modelname{Pythia} and \modelname{GPT-2} families do \emph{not} decrease as the number of trainable parameters in the model increases.
This finding indicates that susceptibility can not be offset by the amount of prior knowledge stored in the language model about a query. 
We also see that \modelname{LLaMA} models do not have lower susceptibility compared with the other two model families, despite their better performance on many downstream tasks \cite{touvron2023llama}.
By comparing instruction-tuned models and base models (i.e., \modelname{LLaMA-2-7B} vs. \modelname{LLaMA-2-7B-chat}, \modelname{LLaMA-3-8B} vs. \modelname{LLaMA-3-8B-Instruct}), we find instruction tuning increases the susceptibility of language models.
These comparisons align with the empirical observation that instruction-tuned models are better at integrating the information in the input context and responding to queries.\looseness=-1

\paragraph{Queries.} 
We also investigate whether language models are more susceptible for some particular types of queries.
Similar to the findings of \citet{du2024context}, \Cref{tab:model-sus-comparison} shows that closed queries (e.g., \querytext{Q: Is Here Comes the Sun performed by The Beatles? A:}) are less susceptible than open queries (e.g., \querytext{Q: Who performed Here Comes the Sun? A:}).
In addition, we investigate two forms of open queries: question-answering (e.g., \querytext{Q: Is Here Comes the Sun performed by The Beatles? A:}) and sentence-completion (e.g., \querytext{Here Comes the Sun is performed by}). 
On \modelname{LLaMA-3-8B-Instruct}, we found that queries in question-answering form have an average Monte Carlo susceptibility of $0.27$ and Fisher susceptibility of $0.04$, while the sentence-completion form has $1.05$ and $0.14$.
Both of the susceptibility metrics show that question-answering queries are less easily affected by context than sentence-completion queries.
This finding supports our claim on the comparability of Fisher susceptibility to Monte Carlo susceptibility.
The full results are given in \cref{fig:qa-vs-completion}.\looseness=-1

\paragraph{Entity Familiarity.}
\citet{du2024context} found, across different models, that real entities tend to be less susceptible than fake ones.
However, we find this pattern does not hold for Fisher susceptibility.
We show the susceptibility comparison for 6 models in \cref{fig:real-vs-fake}.
The Monte Carlo susceptibility of real entities is slightly but consistently lower than that of fake entities.
Meanwhile, the Fisher susceptibility of all models remains similar regardless of whether the entity is real or not.
This could suggest a limitation of Fisher susceptibility as an approximation for susceptibility, which could be due to the assumptions made in \cref{eq:mean-and-variance-1}.\looseness=-1

\section{Conclusion} \label{sec:conclusion}
To efficiently measure a language model's susceptibility, we have proposed Fisher susceptibility, which uses the Fisher information of a language model with regard to its input to measure the scale of distributional changes as the input varies.
Through experiments, we find a strong correlation between a language model's Monte Carlo susceptibility and Fisher susceptibility, which we take to validate our approximation.
Compared to methods that require many context samples and language model forward passes, our method is significantly faster.
Our study contributes to the exploration for interpretable and efficient evaluation metrics for language models.\looseness=-1

\section*{Limitations}
One technical limitation of this work is that we compute an approximation of the Fisher information $\fimsymb(\lmembed(\query))$ using \Cref{eq:fisher-approximation} by taking the top $K$ answers.
Second, computing Fisher susceptibility requires automatic differentiation on the language model, which is more memory intensive (by a factor of 2) than simply performing a forward pass. 

\section*{Ethics Statement}

This paper provides a novel language evaluation metric and experimental analysis on publicly available models and data.
Our ultimate goal in this paper is to contribute to the research on language model interpretability and evaluation.
By investigating the robustness of language models to contexts, we aim to contribute to a research effort of developing more reliable, interpretable models.
We foresee no particular ethical concerns and hope this paper contributes to developing tools that can identify and mitigate ethical concerns in the future.

\bibliography{references/custom.bib,references/anthology.bib}

\clearpage
\appendix
\onecolumn

\section{Derivation of \texorpdfstring{\cref{eq:sus-fi}}{Lg}} \label{app:derivation-sus-fi}

\begin{proposition}
\begin{align}
\stickyscoresymb( \query)  \defequals \MI(\rvContext; \rvAnswer \mid \rvQuery = \query) 
 = \frac{1}{2} \expectation{\rvContext  }{\deltaembedding(\context, \query)^\top  \fimsymb( \lmembed(\query)) \deltaembedding(\context, \query)} + \order{|| \deltaembedding(\context, \query) ||^3 }
\end{align}
where $\deltaembedding(\context, \query) \defequals \lmembed(\context \oplus \query) - \lmembed(\query)$.
\end{proposition}
\begin{proof}
Consider the following manipulation
\begin{subequations}
\begin{align}
 \stickyscoresymb(  \query)   &\defequals \MI(\rvContext; \rvAnswer \mid \rvQuery = \query)  \\ 
 &= \expectation{\rvContext}{\kldiv{p(\rvAnswer \mid \rvContext = \context, \rvQuery = \query)}{p(\rvAnswer \mid \rvQuery = \query)} }\\
 & = \expectation{\rvContext}{\kldiv{\funcp(\cdot \mid \lmembed(\query) + (-\lmembed(\query) + \lmembed(\context \oplus \query))}{\funcp(\cdot \mid \lmembed(\query))}} \\
 & = \frac{1}{2} \expectation{\rvContext  }{(\lmembed(\context) - \lmembed(\query))^\top  \fimsymb( \lmembed(\query)) (\lmembed(\context \oplus \query) - \lmembed(\query))+ \order{||\lmembed(\context \oplus \query) - \lmembed(\query)||^3} } \\
  & = \frac{1}{2} \expectation{\rvContext  }{(\lmembed(\context) - \lmembed(\query))^\top  \fimsymb( \lmembed(\query)) (\lmembed(\context \oplus \query) - \lmembed(\query))} + \expectation{\rvContext  }{\order{||\lmembed(\context \oplus \query) - \lmembed(\query)||^3} }
 \\
 &= \frac{1}{2} \expectation{\rvContext}{(\lmembed(\context) - \lmembed(\query))^\top  \fimsymb( \lmembed(\query)) (\lmembed(\context \oplus \query) - \lmembed(\query))} + \order{\underbrace{\expectation{\rvContext  }{||\lmembed(\context \oplus \query) - \lmembed(\query)||^3}}_{\text{expected approximation error}}} \\
 &= \frac{1}{2} \expectation{\rvContext}{\deltaembedding(\context, \query)^\top  \fimsymb( \lmembed(\query)) \deltaembedding(\context, \query)} + \order{\underbrace{\expectation{\rvContext  }{||\deltaembedding(\context, \query)||^3}}_{\text{expected approximation error}}},
\end{align}
\end{subequations}
which proves the result. 
\end{proof}

\section{Additional Experimental Results} \label{app:additional-results}
We plot the comparison between Fisher susceptibility and Monte Carlo susceptibility in \cref{fig:all-fisher-mi-box} and \cref{fig:all-fisher-mi-reg}.
Across all models, Fisher susceptibility exhibits strong correlation with Monte Carlo susceptibility, which we take to mean that Fisher susceptibility is a good approximation to Monte Carlo susceptibility.

\begin{figure*}
    \centering
    \includegraphics[width=\textwidth]{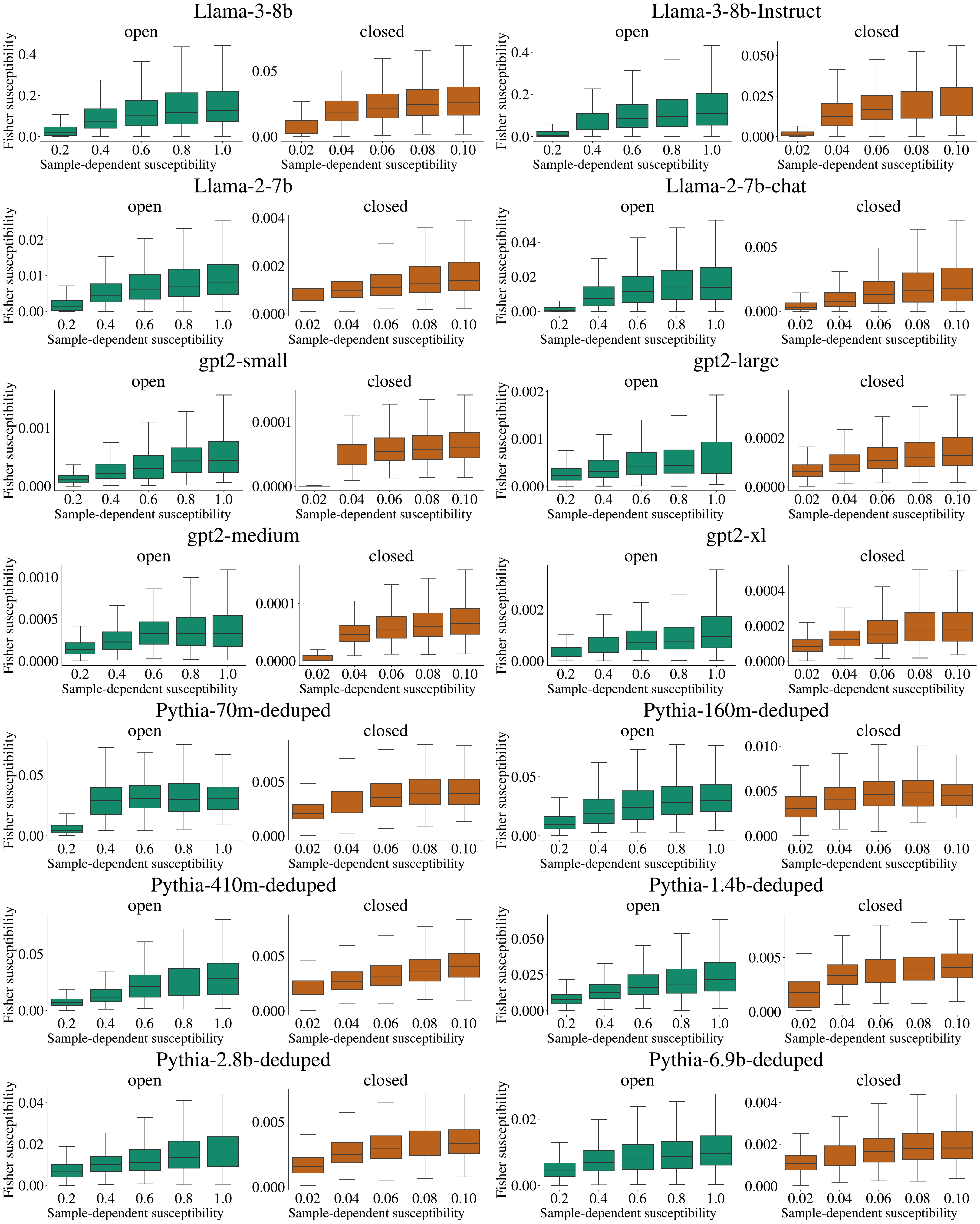}
    \caption{Fisher susceptibility plotted against Monte Carlo susceptibility for different models.}
    \label{fig:all-fisher-mi-box}
\end{figure*}

\begin{figure*}
    \centering
    \includegraphics[width=\textwidth]{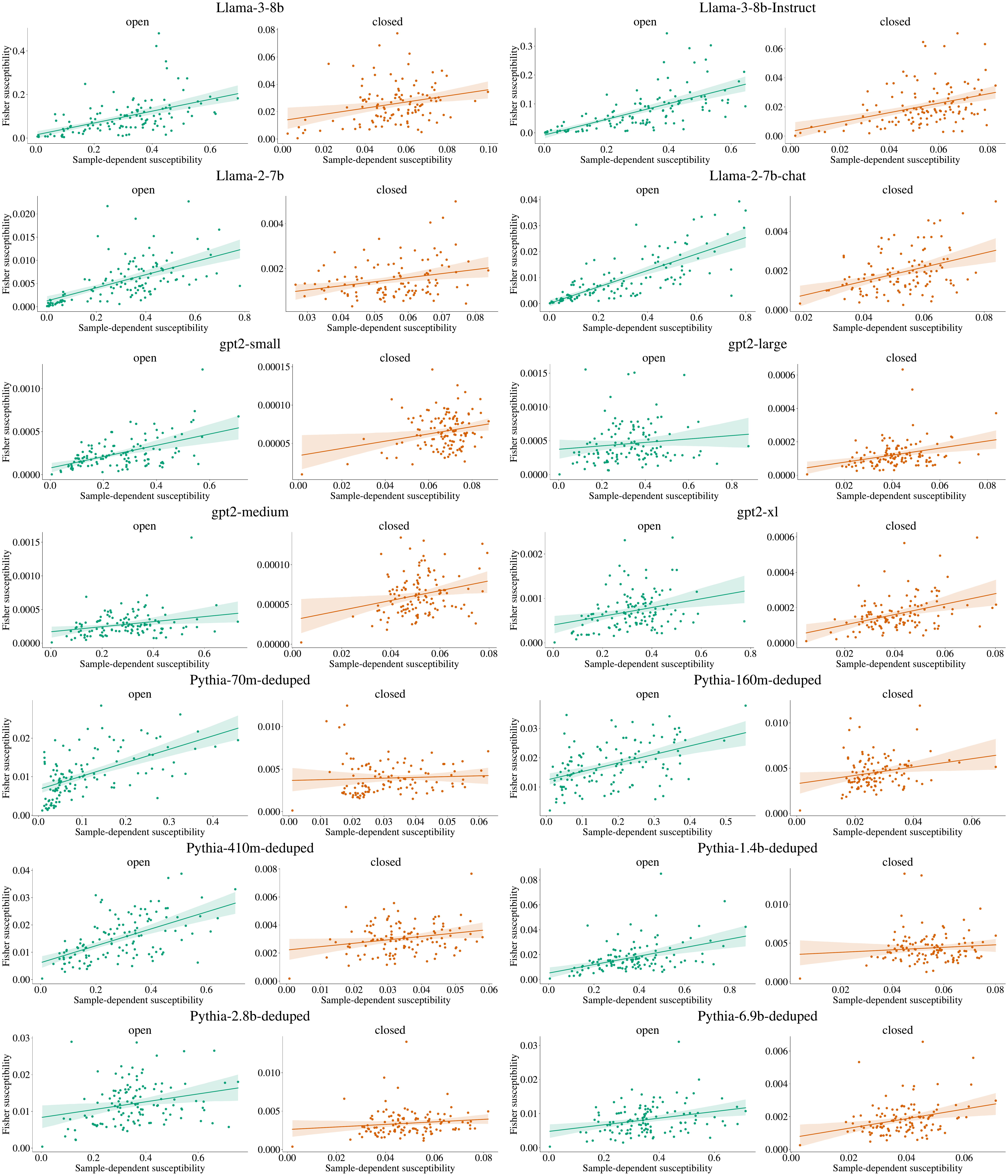}
    \caption{This plot shows the average susceptibility on 122 relation domains in the YAGO dataset for all models. Each point represents the average score of queries on a relation domain. The $x$-coordinate represents Monte Carlo susceptibility and $y$-coordinate represents Fisher susceptibility.}
    \label{fig:all-fisher-mi-reg}
\end{figure*}

\end{document}